\documentclass[11pt]{article}

\usepackage{amsmath,amsthm,amsfonts,latexsym,amssymb,amscd}
\usepackage{bm}
\usepackage{braket}
\usepackage{physics}

\usepackage{graphicx}
\usepackage{tikz}
\usetikzlibrary{chains, automata, positioning, arrows, arrows.meta, calc}
\usepackage[linguistics]{forest}

\usepackage{array}
\usepackage{booktabs}

\usepackage{amsmath,amsthm,amsfonts,latexsym,amssymb,amscd}
\usepackage{bm}
\usepackage{braket}
\usepackage{physics}

\usepackage{graphicx}
\usepackage{tikz}
\usetikzlibrary{chains, automata, positioning, arrows, arrows.meta, calc}
\usepackage[linguistics]{forest}

\usepackage{array}
\usepackage{booktabs}
\usepackage{multirow}

\usepackage{algorithm}
\usepackage{algorithmic}

\usepackage[margin=1in]{geometry}
\usepackage{setspace}
\usepackage{times}
\usepackage{enumitem}
\usepackage{float}
\usepackage{caption}

\usepackage{xcolor}
\usepackage{mdframed}
\definecolor{mypink2}{RGB}{219, 48, 122}

\usepackage{url}
\usepackage{hyperref}
\usepackage{cleveref}
\usepackage{natbib}

\usepackage{faktor}
\usepackage{listings}
\usepackage{authblk}

\theoremstyle{plain}
\newtheorem{thm}{Theorem}[section]

\newtheorem{prp}[thm]{Proposition}

\newtheorem*{clm*}{Claim}

\theoremstyle{definition}

\newtheorem{remrk}[thm]{Remark}

\numberwithin{equation}{section}

\newcommand{\cproof}{\noindent{\it Proof of Claim.}\ }
\newcommand{\cqed}{\hfill\rule{1.3mm}{3mm}}

\newcommand{\jf}{\color{black}}
\newcommand{\mm}{\color{black}}

\newlength\myindent
\newcounter{constant}

\newmdenv[linewidth=2pt,roundcorner=10pt,backgroundcolor=white,font=\bfseries]{questionbox}

\begin{document}
\title{The Normalized Difference Layer: A Differentiable Spectral Index Formulation for Deep Learning}

\author[1]{Ali Lotfi}
\author[2]{Adam Carter}
\author[3,*]{Mohammad Meysami}
\author[4]{Thuan Ha}
\author[5]{Kwabena Nketia}
\author[6]{Steve Shirtliffe}

\affil[1]{\normalsize Nutrien Centre for Sustainable and Digital Agriculture, Department of Plant Sciences,\\ University of Saskatchewan, Saskatoon, SK, Canada\\
\texttt{all054@mail.usask.ca}}
\affil[2]{\normalsize Crop Development Centre, Department of Plant Sciences,\\ University of Saskatchewan, Saskatoon, SK, Canada\\
\texttt{adam.carter@usask.ca}}
\affil[3]{\normalsize Department of Mathematics, The University of Tulsa, Tulsa, OK, USA\\
\texttt{mohammad-meysami@utulsa.edu}}
\affil[4]{\normalsize Nutrien Centre for Sustainable and Digital Agriculture, Department of Plant Sciences,\\ University of Saskatchewan, Saskatoon, SK, Canada\\
\texttt{thuan.ha@usask.ca}}
\affil[5]{\normalsize Nutrien Centre for Sustainable and Digital Agriculture, Department of Plant Sciences,\\ University of Saskatchewan, Saskatoon, SK, Canada\\
\texttt{kwabena.nketia@usask.ca}}
\affil[6]{\normalsize Nutrien Centre for Sustainable and Digital Agriculture, Department of Plant Sciences,\\ University of Saskatchewan, Saskatoon, SK, Canada\\
\texttt{steve.shirtliffe@usask.ca}}
\affil[*]{\normalsize Corresponding author: Ali Lotfi, \texttt{all054@usask.ca}}

\date{}

\maketitle

\begin{abstract}
{\mm Normalized difference indices have been a staple in remote sensing for decades. They stay reliable under lighting changes produce bounded values and connect well to biophysical signals. Even so, they are usually treated as a fixed pre processing step with coefficients set to one, which limits how well they can adapt to a specific learning task. In this study, we introduce the Normalized Difference Layer that is a differentiable neural network module. The proposed method keeps the classical idea but learns the band coefficients from data. We present a complete mathematical framework for integrating this layer into deep learning architectures that uses softplus reparameterization to ensure positive coefficients and bounded denominators. We describe forward and backward pass algorithms enabling end to end training through backpropagation. This approach preserves the key benefits of normalized differences, namely illumination invariance and outputs bounded to $[-1,1]$ while allowing gradient descent to discover task specific band weightings. We extend the method to work with signed inputs, so the layer can be stacked inside larger architectures. Experiments show that models using this layer reach similar classification accuracy to standard multilayer perceptrons while using about 75\% fewer parameters. They also handle multiplicative noise well, at 10\% noise accuracy drops only 0.17\% versus 3.03\% for baseline MLPs. The learned coefficient patterns stay consistent across different depths.}
\end{abstract}

\section{Introduction}

{\mm 
The monitoring of vegetation in arable land was revolutionized by remote sensing. It provides the capacity to carry out systematic surveillance of large areas where this would be impractical based only on field surveys \citep{xue2017significant}. Normalized difference indices, which are easy to be computed and with a clear physical interpretation \citep{bannari1995review}, is one of the most-used methods in processing multispectral data. A normalized difference index is obtained by subtracting two bands and dividing by the sum. This simple index has many practical advantages for agriculture; it is invariant to variations of lighting, is less sensitive to topographic shading and is significantly more robust under diverse sensor calibrations \citep{bannari1995review}.

The NDVI is the first index proposed by Rouse et al. \ (1973) at Texas A\&M University for the Great Plains Corridor Project and soon thereafter became a textbook example of an elementary normalized index \citep{tucker1979red}. The index was developed as a correction for the confusing effects of changing solar zenith angle on satellite data covering different latitude ranges. NDVI uses the well-known spectral behavior of photosynthetically active vegetation, which absorbs most of the red light for chlorophyll and reflects strongly near infrared light by cellular structure \citep{tucker1979red}. Using the normalized ratio of NIR to red reflectance, NDVI mitigates atmospheric influence and is a reasonably globally stable vegetation greenness measure. This single variable is demonstrated to be closely related to critical biophysical parameters, including the leaf area index (LAI), chlorophyll concentration, and fraction of vegetation cover \citep{huete2002overview}.

The way band ratios and normalized differences are constructed tends to bring out the composition we care about while pushing down many confusing effects in the scene, for example, terrain slope or changes in grain size \citep{unger2007introductory}. This behavior has been useful in precision agriculture, where farmers and researchers seek indicators of crop health they can trust and remain roughly stable across different illumination, sun angles, and atmospheric conditions \citep{haboudane2004hyperspectral}. Over time, the number of NDVI-related publications increased from 795 in the 1990s to more than 12,618 in the 2010s, indicating that normalized indices are now widely used in agricultural and environmental monitoring work \citep{tang2015drone}.

The success of NDVI has naturally encouraged the development of many other spectral indices, each designed either to overcome specific shortcomings or to make better use of certain spectral characteristics. One example is the Soil Adjusted Vegetation Index (SAVI), proposed by Huete (1988), which includes a soil brightness term to reduce the influence of background soil reflectance in areas where vegetation cover is limited \citep{huete1988soil}. The Enhanced Vegetation Index (EVI) builds on this idea by incorporating blue band reflectance to correct for remaining atmospheric effects and to limit saturation in regions with very dense biomass \citep{huete1994development}.  In more recent years, the spread of sensors with red edge bands, especially the Sentinel 2 Multi-Spectral Instrument (MSI), has led to indices such as the Normalized Difference Red Edge Index (NDRE), which are more sensitive to changes in chlorophyll content and plant nitrogen status \citep{sonobe2018crop}. These spectral indices are now widely used in crop classification and in agricultural monitoring. The red edge bands that are specific to Sentinel 2, sitting between the red and near infrared parts of the spectrum, work especially well for picking up small changes in crop canopy traits and for separating vegetation species at fairly fine taxonomic levels \citep{immitzer2016first, kang2021crop}.

Despite these advances in general vegetation monitoring, some invasive species remain difficult to identify from satellite imagery. For Kochia (\textit{Bassia scoparia}), reliably spotting the plants remains a challenge because, in many fields, they share morphological and spectral traits with surrounding crops, leading to signal mixing. Work with hyperspectral imagery has reported classification accuracies of 67--80\% for separating herbicide-resistant Kochia biotypes using support vector machine classifiers \citep{nugent2018discrimination, scherrer2019hyperspectral}. More recent studies have tried attention-based convolutional neural networks and have pushed the accuracy above 99\% when separating Kochia from sugarbeet under field conditions \citep{mensah2025detection}. These deep learning approaches are powerful, but they usually depend on costly hyperspectral sensors and fairly involved computational setups, which makes them awkward to use for everyday regional monitoring with freely available Sentinel-2 imagery. Deep learning methods that use UAV imagery and convolutional neural networks have shown they can map weeds in wheat fields accurately, even across large orthomosaic images \citep{wang2023weed}. One ongoing challenge is the time-consuming process of creating labeled training datasets. Recent studies have tackled this by using automated feature-labeling workflows, which can segment kochia and other weeds from UAV images with 87\% accuracy, all without manual training samples \citep{ha2025automated}.

Deep learning methods often achieve very strong classification performance in remote sensing, but bringing them into day to day agricultural use is still challenging because their decisions are hard to interpret. A typical deep neural network behaves like a black box, which makes it difficult to understand why the model favors one class over another \citep{reichstein2019deep, sarker2021deep}. In earth observation practice, many users care about more than accuracy alone. They also want ecological insight and a clearer sense of the processes behind the signal, and this lack of transparency can get in the way of that goal \citep{mcgovern2019making}. When models are built from polynomial functions of familiar spectral indices, the link to biophysical processes is much clearer, so experts can check whether the model behavior agrees with established scientific knowledge \citep{de2000classification}. By contrast, many deep learning models, although very good at capturing intricate patterns in the spectral and spatial domain, still behave largely as black boxes and make it difficult to interpret the ecological processes behind their predictions \citep{maxwell2021accuracy}. Recent studies suggest that systematically exploring polynomial combinations of normalized differences can surface compact and interpretable spectral indices for vegetation classification. In the case of Kochia detection with Sentinel 2 imagery, a single degree 2 index built from red edge bands reaches 96.26\% accuracy, where the key signal appears to come from interactions between spectral bands rather than from individual band ratios \citep{lotfi2025automated}.

In the last few years, more people in machine learning have been trying to build domain knowledge into neural networks, rather than relying only on patterns learned from data. \citep{raissi2019physics} in early work showed that putting physical laws directly into training as rules can help models generalize better and make smarter use of limited data. In this setup, prior knowledge is not just a step you do before training, it becomes part of how the network is shaped and how it learns. The basic idea is simple, if a model is designed to respect known physical relationships, it often needs less training data. This tends to handle new situations more smoothly, and the outputs usually look more consistent with physical reasoning.

Inductive bias is basically the set of assumptions a learning algorithm leans on when it has to make predictions for inputs it has not seen before and it has become a big part of how modern deep learning models are designed. Different architectures build in different assumptions. Convolutional networks, for example, lean on translation equivariance. Recurrent networks assume there are sequential dependencies. And attention mechanisms lean on the idea that relevance can be inferred from similarity. Every design choice in a model narrows the hypothesis space, and when those assumptions line up with the problem, learning can get a lot more efficient. In remote sensing, the normalized difference formulation is a good example of this kind of built in domain knowledge. It bakes in illumination invariance, keeps outputs in a bounded range that works well for classification, and captures ratio based relationships that research over many years has found to be useful for vegetation analysis.

Normalized difference indices are used all over remote sensing, but they still have not been built into neural network architectures as much as you might expect. In most setups spectral indices are handled as a fixed preprocessing step. We calculate NDVI, NDWI, or related indices from the raw bands and then feed those derived features into a conventional classifier. This two-stage pipeline has an important limitation---the coefficients weighting each band in the normalized difference formula are fixed at unity by convention, not optimized for the specific classification task at hand. There is no strong reason the most useful discriminative signal has to come from the symmetric form $\left(b_i - b_j\right)/\left(b_i + b_j\right)$, instead of an asymmetric combination that leans more on one band than the other. In reality, the best weighting is probably tied to the task. Finding stressed vegetation may call for different coefficients than identifying invasive species, and weights that look good on one sensor might not translate well to another.

This point leads directly to the main contribution of this work. We embed the normalized difference formulation inside the neural network as a differentiable layer with learnable coefficients. Instead of choosing band weights by hand or computing indices as a separate preprocessing step, we let gradient descent learn the best weights end to end through backpropagation. The resulting Normalized Difference Layer keeps the illumination invariance and the bounded output range that make the classic formulation so useful, but it also learns coefficient values that fit the specific classification task. More broadly, this is an example of a common idea in scientific machine learning. If you build well understood physical structure directly into neural network layers, you can often get better efficiency and more robustness than you would from an unconstrained architecture.

We evaluate the proposed architecture on Kochia detection using Sentinel-2 imagery from Saskatchewan, Canada. Kochia is an invasive weed of significant agricultural concern in North American prairies, and distinguishing it from crops presents a challenging classification problem due to spectral similarities during critical growth stages. In our experiments, we compare networks that use the learnable Normalized Difference Layer with standard multilayer perceptrons that have similar representational capacity. We look at classification accuracy, parameter efficiency, robustness to noise, and how interpretable the learned coefficients are. Overall, the domain informed architecture reaches competitive accuracy while using far fewer parameters, and it shows strong robustness to the multiplicative noise that is common in satellite imagery.

The structure of this paper is as follows. In Section~\ref{sec:methodology}, we provide the mathematical framework for the Normalized Difference Layer. It includes the gradient derivations needed for backpropagation, and it describes the algorithms for the forward and backward passes. In Section~\ref{sec:experiments} we present our experimental evaluation on the Kochia detection task and compare the proposed architecture against baseline models across several metrics, including accuracy, parameter efficiency, noise robustness, and convergence behavior. We also examine the learned coefficients to show the interpretability advantages of the proposed approach. In Section~\ref{sec:conclusion} we summarize our findings and point to directions for future work.

}

\section{Methodology}\label{sec:methodology}

Consider the Normalized Difference Vegetation Index (NDVI), perhaps the most widely used spectral index in remote sensing:
\begin{equation}
\text{NDVI} = \frac{b_{\text{NIR}} - b_{\text{Red}}}{b_{\text{NIR}} + b_{\text{Red}}}
\label{eq:ndvi}
\end{equation}
{\mm where $b_{\text{NIR}}$ and $b_{\text{Red}}$ denote the reflectance values in the near-infrared and red spectral bands, respectively. Why does this formula work so well for detecting vegetation? The key point is that it stays stable under changes in illumination. For example, suppose the true reflectance values are all multiplied by some factor $k$ because the sun angle changes, clouds create shadows, or the sensor response shifts a bit, all of which happens a lot in satellite imagery. Observe that:
}
\begin{equation}
\frac{k \cdot b_{\text{NIR}} - k \cdot b_{\text{Red}}}{k \cdot b_{\text{NIR}} + k \cdot b_{\text{Red}}} = \frac{b_{\text{NIR}} - b_{\text{Red}}}{b_{\text{NIR}} + b_{\text{Red}}}
\end{equation}
{\jf The scaling factor $k$ cancels out completely. This is what makes NDVI dependable when overall brightness changes would otherwise throw off a classifier trained on raw reflectance values. Also, the ratio form keeps the output bounded in $[-1, 1]$, which works nicely for downstream learning algorithms. For decades, remote sensing scientists have used the same formula and just swapped the band pair to focus on different phenomena, like NDWI for water, NDBI for built up areas, NBR for burn severity, and plenty of others. Each of these indices shares the same normalized difference structure---only the choice of bands changes.}

But look again at Equation~\eqref{eq:ndvi}. There is an implicit assumption hiding in plain sight: both bands receive equal weight. We could write this explicitly as:
\begin{equation}
\text{NDVI} = \frac{1 \cdot b_{\text{NIR}} - 1 \cdot b_{\text{Red}}}{1 \cdot b_{\text{NIR}} + 1 \cdot b_{\text{Red}}}
\end{equation}
{\mm Why should the coefficients be exactly 1? For a general classification task, it is easy to imagine that the best discriminative signal is not in the symmetric difference, but in an \emph{asymmetric} combination, maybe weighting NIR more than Red, or the other way around. The best weighting will likely depend on the problem. Detecting stressed vegetation may benefit from different coefficients than detecting invasive species.

This observation points to a simple question: \emph{what if we let the data choose the coefficients?} Instead of picking them by hand, we just let the model learn the weights during training, alongside everything else, using gradient descent. Then we take that idea and build it into the architecture by putting the normalized difference formulation inside the network as a differentiable layer with learnable coefficients. Unlike the usual approach where spectral indices are computed as a fixed preprocessing step, our architecture learns the band weights during training through backpropagation. At the same time, it keeps the illumination invariance and the bounded output that make normalized differences work so well.

}
\subsection{Normalized Difference Layer with Coupling Coefficients}

The core building block is a \textit{Normalized Difference Layer} (ND Layer), where each node computes a weighted normalized difference between two spectral bands. For bands $b_i$ and $b_j$, the output is defined as:
\begin{equation}
N_{ij} = \frac{\operatorname{softplus}(\alpha_{ij})\, b_i - \operatorname{softplus}(\beta_{ij})\, b_j}{\operatorname{softplus}(\alpha_{ij})\, b_i + \operatorname{softplus}(\beta_{ij})\, b_j + \epsilon}
\label{eq:nd}
\end{equation}
where $\alpha_{ij}, \beta_{ij} \in \mathbb{R}$ are learnable parameters, $\operatorname{softplus}(x) = \log(1 + e^x)$ ensures positivity, and $\epsilon > 0$ is a small constant for numerical stability.

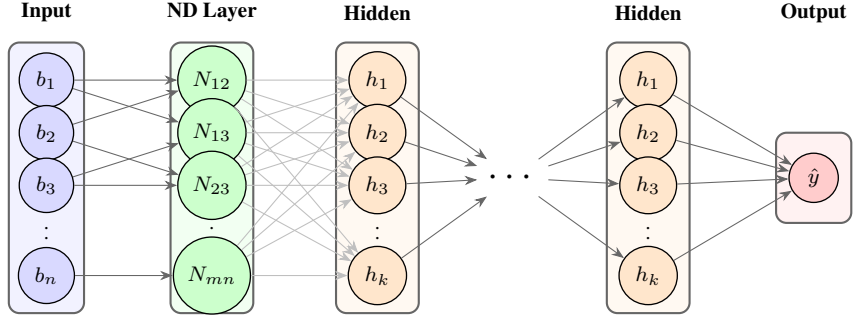
\begin{figure}[H]
\centering
\begin{tikzpicture}[
    node distance=0.5cm and 1.5cm,
    band/.style={
        circle, 
        draw=black, 
        fill=blue!15, 
        minimum size=0.6cm, 
        font=\scriptsize,
        line width=0.5pt
    },
    nd/.style={
        circle, 
        draw=black, 
        fill=green!20, 
        minimum size=0.6cm, 
        font=\scriptsize,
        line width=0.5pt
    },
    hidden/.style={
        circle, 
        draw=black, 
        fill=orange!20, 
        minimum size=0.6cm, 
        font=\scriptsize,
        line width=0.5pt
    },
    output/.style={
        circle, 
        draw=black, 
        fill=red!20, 
        minimum size=0.6cm, 
        font=\scriptsize,
        line width=0.5pt
    },
    conn/.style={->, >=Stealth, line width=0.4pt, black!60},
    layerlabel/.style={font=\scriptsize\bfseries},
    layerbox/.style={
        draw=black!60, 
        rounded corners=4pt, 
        line width=0.8pt,
        inner sep=0.2cm
    }
]

\def\ytop{0}
\def\ysecond{-0.7}
\def\ythird{-1.4}
\def\ydots{-2.0}
\def\ybottom{-2.6}

\def\xinput{0}
\def\xnd{2.2}
\def\xhidden{4.4}
\def\xdots{6.2}
\def\xhiddenlast{8.0}
\def\xoutput{10.2}

\draw[layerbox, fill=blue!5] 
    (\xinput-0.5, \ytop+0.5) rectangle (\xinput+0.5, \ybottom-0.5);
\node[layerlabel] at (\xinput, \ytop+0.9) {Input};

\draw[layerbox, fill=green!5] 
    (\xnd-0.55, \ytop+0.5) rectangle (\xnd+0.55, \ybottom-0.5);
\node[layerlabel] at (\xnd, \ytop+0.9) {ND Layer};

\draw[layerbox, fill=orange!5] 
    (\xhidden-0.55, \ytop+0.5) rectangle (\xhidden+0.55, \ybottom-0.5);
\node[layerlabel] at (\xhidden, \ytop+0.9) {Hidden};

\draw[layerbox, fill=orange!5] 
    (\xhiddenlast-0.55, \ytop+0.5) rectangle (\xhiddenlast+0.55, \ybottom-0.5);
\node[layerlabel] at (\xhiddenlast, \ytop+0.9) {Hidden};

\draw[layerbox, fill=red!5] 
    (\xoutput-0.5, -0.7) rectangle (\xoutput+0.5, -1.9);
\node[layerlabel] at (\xoutput, \ytop+0.9) {Output};

\node[band] (b1) at (\xinput, \ytop) {$b_1$};
\node[band] (b2) at (\xinput, \ysecond) {$b_2$};
\node[band] (b3) at (\xinput, \ythird) {$b_3$};
\node[font=\scriptsize] (bdots) at (\xinput, \ydots) {$\vdots$};
\node[band] (bn) at (\xinput, \ybottom) {$b_n$};

\node[nd] (n12) at (\xnd, \ytop) {$N_{12}$};
\node[nd] (n13) at (\xnd, \ysecond) {$N_{13}$};
\node[nd] (n23) at (\xnd, \ythird) {$N_{23}$};
\node[font=\scriptsize] (ndots) at (\xnd, \ydots) {$\vdots$};
\node[nd] (nmn) at (\xnd, \ybottom) {$N_{mn}$};

\node[hidden] (h11) at (\xhidden, \ytop) {$h_1$};
\node[hidden] (h12) at (\xhidden, \ysecond) {$h_2$};
\node[hidden] (h13) at (\xhidden, \ythird) {$h_3$};
\node[font=\scriptsize] (h1dots) at (\xhidden, \ydots) {$\vdots$};
\node[hidden] (h1k) at (\xhidden, \ybottom) {$h_k$};

\node[font=\Large] (layerdots) at (\xdots, -1.3) {$\cdots$};

\node[hidden] (h21) at (\xhiddenlast, \ytop) {$h_1$};
\node[hidden] (h22) at (\xhiddenlast, \ysecond) {$h_2$};
\node[hidden] (h23) at (\xhiddenlast, \ythird) {$h_3$};
\node[font=\scriptsize] (h2dots) at (\xhiddenlast, \ydots) {$\vdots$};
\node[hidden] (h2k) at (\xhiddenlast, \ybottom) {$h_k$};

\node[output] (y) at (\xoutput, -1.3) {$\hat{y}$};

\draw[conn] (b1) -- (n12);
\draw[conn] (b2) -- (n12);
\draw[conn] (b1) -- (n13);
\draw[conn] (b3) -- (n13);
\draw[conn] (b2) -- (n23);
\draw[conn] (b3) -- (n23);
\draw[conn] (bn) -- (nmn);

\foreach \i in {n12, n13, n23, nmn} {
    \foreach \j in {h11, h12, h13, h1k} {
        \draw[conn, black!25] (\i) -- (\j);
    }
}

\foreach \i in {h11, h12, h13, h1k} {
    \draw[conn] (\i) -- (layerdots);
}

\foreach \j in {h21, h22, h23, h2k} {
    \draw[conn] (layerdots) -- (\j);
}

\foreach \i in {h21, h22, h23, h2k} {
    \draw[conn] (\i) -- (y);
}

\end{tikzpicture}
\caption{Architecture of the proposed deep network with Normalized Difference Layer. The first hidden layer is a Normalized Difference Layer, where each node $N_{ij}$ receives exactly two spectral bands and computes a weighted normalized difference with learnable coefficients $\alpha_{ij}$ and $\beta_{ij}$, outputting a bounded value in $[-1, 1]$.}
\label{fig:architecture}
\end{figure}

{\mm 
When $\alpha_{ij} = \beta_{ij} = 0$, Equation~\eqref{eq:nd} is just the usual normalized difference used in indices like NDVI. Once the network learns these parameters, it can adjust the band weights for the task at hand, and it still keeps the output in $[-1, 1]$ and keeps the same illumination invariance you get from the normalized difference form.
}

\subsection{Gradient Derivation}

To enable backpropagation through the Normalized Difference Layer, we derive the partial derivatives of $N_{ij}$ with respect to the learnable parameters $\alpha_{ij}$, $\beta_{ij}$ and the inputs $b_i$, $b_j$.

\begin{prp}[Gradients of the Normalized Difference Layer]
\label{prp:gradients}
Let $\sigma_\alpha = \operatorname{softplus}(\alpha_{ij})$, $\sigma_\beta = \operatorname{softplus}(\beta_{ij})$, and $B = \sigma_\alpha b_i + \sigma_\beta b_j + \epsilon$. The partial derivatives of $N_{ij}$ defined in Equation~\eqref{eq:nd} are:
\begin{align}
\frac{\partial N_{ij}}{\partial \alpha_{ij}} &= \frac{\operatorname{sigmoid}(\alpha_{ij})\, b_i \left(2\sigma_\beta b_j + \epsilon\right)}{B^2} \label{eq:dN_dalpha} \\[0.8em]
\frac{\partial N_{ij}}{\partial \beta_{ij}} &= \frac{-\operatorname{sigmoid}(\beta_{ij})\, b_j \left(2\sigma_\alpha b_i + \epsilon\right)}{B^2} \label{eq:dN_dbeta} \\[0.8em]
\frac{\partial N_{ij}}{\partial b_i} &= \frac{\sigma_\alpha \left(2\sigma_\beta b_j + \epsilon\right)}{B^2} \label{eq:dN_dbi} \\[0.8em]
\frac{\partial N_{ij}}{\partial b_j} &= \frac{-\sigma_\beta \left(2\sigma_\alpha b_i + \epsilon\right)}{B^2} \label{eq:dN_dbj}
\end{align}
\end{prp}

\begin{proof}
Let $A = \sigma_\alpha b_i - \sigma_\beta b_j$ denote the numerator of $N_{ij}$, so that $N_{ij} = A/B$. Recall that:
\begin{equation}
\frac{d}{d\alpha_{ij}}\operatorname{softplus}(\alpha_{ij}) = \operatorname{sigmoid}(\alpha_{ij}), \quad \frac{d}{d\beta_{ij}}\operatorname{softplus}(\beta_{ij}) = \operatorname{sigmoid}(\beta_{ij})
\end{equation}
The partial derivatives of $A$ are:
\begin{align}
\frac{\partial A}{\partial \alpha_{ij}} &= \operatorname{sigmoid}(\alpha_{ij})\, b_i \label{eq:dA_dalpha} \\
\frac{\partial A}{\partial \beta_{ij}} &= -\operatorname{sigmoid}(\beta_{ij})\, b_j \label{eq:dA_dbeta} \\
\frac{\partial A}{\partial b_i} &= \sigma_\alpha \label{eq:dA_dbi} \\
\frac{\partial A}{\partial b_j} &= -\sigma_\beta \label{eq:dA_dbj}
\end{align}
The partial derivatives of $B$ are:
\begin{align}
\frac{\partial B}{\partial \alpha_{ij}} &= \operatorname{sigmoid}(\alpha_{ij})\, b_i \label{eq:dB_dalpha} \\
\frac{\partial B}{\partial \beta_{ij}} &= \operatorname{sigmoid}(\beta_{ij})\, b_j \label{eq:dB_dbeta} \\
\frac{\partial B}{\partial b_i} &= \sigma_\alpha \label{eq:dB_dbi} \\
\frac{\partial B}{\partial b_j} &= \sigma_\beta \label{eq:dB_dbj}
\end{align}

For $\xi \in \{\alpha_{ij}, \beta_{ij}, b_i, b_j\}$, the quotient rule gives:
\begin{equation}
\frac{\partial N_{ij}}{\partial \xi} = \frac{1}{B^2}\left(B\frac{\partial A}{\partial \xi} - A\frac{\partial B}{\partial \xi}\right)
\label{eq:quotient}
\end{equation}
We note two following two identities to simplify the derivatives:
\begin{align}
B - A &= 2\sigma_\beta b_j + \epsilon \label{eq:BminusA} \\
B + A &= 2\sigma_\alpha b_i + \epsilon \label{eq:BplusA}
\end{align}

For $\partial N_{ij}/\partial \alpha_{ij}$: by Equations~\eqref{eq:dA_dalpha} and~\eqref{eq:dB_dalpha}, we have $\frac{\partial A}{\partial \alpha_{ij}} = \frac{\partial B}{\partial \alpha_{ij}} = \operatorname{sigmoid}(\alpha_{ij})\, b_i$. Substituting into~\eqref{eq:quotient} and applying~\eqref{eq:BminusA} yields~\eqref{eq:dN_dalpha}.

For $\partial N_{ij}/\partial \beta_{ij}$: by Equations~\eqref{eq:dA_dbeta} and~\eqref{eq:dB_dbeta}, substituting into~\eqref{eq:quotient} and applying~\eqref{eq:BplusA} yields~\eqref{eq:dN_dbeta}.

For $\partial N_{ij}/\partial b_i$: by Equations~\eqref{eq:dA_dbi} and~\eqref{eq:dB_dbi}, we have $\frac{\partial A}{\partial b_i} = \frac{\partial B}{\partial b_i} = \sigma_\alpha$. Substituting into~\eqref{eq:quotient} and applying~\eqref{eq:BminusA} yields~\eqref{eq:dN_dbi}.

For $\partial N_{ij}/\partial b_j$: by Equations~\eqref{eq:dA_dbj} and~\eqref{eq:dB_dbj}, substituting into~\eqref{eq:quotient} and applying~\eqref{eq:BplusA} yields~\eqref{eq:dN_dbj}.
\end{proof}

\subsection{Backpropagation Algorithm}

{\mm We present the backpropagation steps for the Normalized Difference Layer so it can be trained end to end with gradient descent. Let $\mathcal{L}$ denote the loss function. In the backward pass, we get $\delta_{ij} = \frac{\partial \mathcal{L}}{\partial N_{ij}}$ from the layers above. From there, we compute (1) gradients with respect to the learnable parameters $\alpha_{ij}$ and $\beta_{ij}$ for the weight updates, and (2) gradients with respect to the inputs $b_i$ so they can be passed back to the earlier layers.}

\begin{algorithm}
\caption{Forward Pass: Normalized Difference Layer}
\label{alg:forward}
\begin{algorithmic}
\REQUIRE Input bands $\mathbf{b} = (b_1, \ldots, b_n)$, parameters $\{\alpha_{ij}, \beta_{ij}\}_{i<j}$, stability constant $\epsilon$
\ENSURE ND outputs $\{N_{ij}\}_{i<j}$, cached values for backward pass
\FOR{each pair $(i, j)$ with $i < j$}
    \STATE $\sigma_\alpha \gets \log(1 + e^{\alpha_{ij}})$ \COMMENT{softplus}
    \STATE $\sigma_\beta \gets \log(1 + e^{\beta_{ij}})$
    \STATE $A_{ij} \gets \sigma_\alpha \cdot b_i - \sigma_\beta \cdot b_j$ \COMMENT{numerator}
    \STATE $B_{ij} \gets \sigma_\alpha \cdot b_i + \sigma_\beta \cdot b_j + \epsilon$ \COMMENT{denominator}
    \STATE $N_{ij} \gets A_{ij} / B_{ij}$
    \STATE \textbf{cache} $(\sigma_\alpha, \sigma_\beta, b_i, b_j, B_{ij})$ \COMMENT{store for backward pass}
\ENDFOR
\RETURN $\{N_{ij}\}_{i<j}$
\end{algorithmic}
\end{algorithm}

\begin{algorithm}[H]
\caption{Backward Pass: Normalized Difference Layer}
\label{alg:backward}
\begin{algorithmic}
\REQUIRE Upstream gradients $\{\delta_{ij}\}_{i<j}$ where $\delta_{ij} = \frac{\partial \mathcal{L}}{\partial N_{ij}}$, cached values from forward pass
\ENSURE Parameter gradients $\{\frac{\partial \mathcal{L}}{\partial \alpha_{ij}}, \frac{\partial \mathcal{L}}{\partial \beta_{ij}}\}_{i<j}$, input gradients $\{\frac{\partial \mathcal{L}}{\partial b_k}\}_{k=1}^{n}$
\STATE Initialize $\frac{\partial \mathcal{L}}{\partial b_k} \gets 0$ for all $k \in \{1, \ldots, n\}$
\FOR{each pair $(i, j)$ with $i < j$}
    \STATE \textbf{retrieve} $(\sigma_\alpha, \sigma_\beta, b_i, b_j, B_{ij})$ from cache
    \STATE $s_\alpha \gets \frac{1}{1 + e^{-\alpha_{ij}}}$ \COMMENT{sigmoid}
    \STATE $s_\beta \gets \frac{1}{1 + e^{-\beta_{ij}}}$
    \STATE $B^2 \gets B_{ij}^2$
    \STATE \COMMENT{Parameter gradients (Proposition~\ref{prp:gradients})}
    \STATE $\frac{\partial \mathcal{L}}{\partial \alpha_{ij}} \gets \delta_{ij} \cdot \frac{s_\alpha \cdot b_i \cdot (2\sigma_\beta b_j + \epsilon)}{B^2}$
    \STATE $\frac{\partial \mathcal{L}}{\partial \beta_{ij}} \gets \delta_{ij} \cdot \frac{-s_\beta \cdot b_j \cdot (2\sigma_\alpha b_i + \epsilon)}{B^2}$
    \STATE \COMMENT{Input gradients (accumulate across all pairs)}
    \STATE $\frac{\partial \mathcal{L}}{\partial b_i} \gets \frac{\partial \mathcal{L}}{\partial b_i} + \delta_{ij} \cdot \frac{\sigma_\alpha \cdot (2\sigma_\beta b_j + \epsilon)}{B^2}$
    \STATE $\frac{\partial \mathcal{L}}{\partial b_j} \gets \frac{\partial \mathcal{L}}{\partial b_j} + \delta_{ij} \cdot \frac{-\sigma_\beta \cdot (2\sigma_\alpha b_i + \epsilon)}{B^2}$
\ENDFOR
\RETURN $\left\{\frac{\partial \mathcal{L}}{\partial \alpha_{ij}}, \frac{\partial \mathcal{L}}{\partial \beta_{ij}}\right\}_{i<j}$, $\left\{\frac{\partial \mathcal{L}}{\partial b_k}\right\}_{k=1}^{n}$
\end{algorithmic}
\end{algorithm}

{\jf The backward pass applies the chain rule repeatedly. At each ND node, we take the upstream gradient $\delta_{ij}$ and scale it by the local derivatives given in Proposition~\ref{prp:gradients}. The input gradients $\frac{\partial \mathcal{L}}{\partial b_k}$ accumulate contributions from all pairs $(ij)$ that involve band $b_k$, since each input band participates in $(n-1)$ normalized difference computations.

\begin{remrk}[Computational Complexity]
For $n$ input bands, the ND Layer computes $\binom{n}{2} = \frac{n(n-1)}{2}$ pairwise outputs. Both the forward and backward passes have time complexity $O(n^2)$. The space complexity is also $O(n^2)$ due to caching intermediate values. In practice for typical spectral band counts ($n \leq 15$), this overhead is small compared to the fully connected layers that follow.
\end{remrk} 
These gradients have desirable properties for optimization. Since $\sigma_\alpha, \sigma_\beta > 0$ and $\epsilon > 0$, we get $B > \epsilon > 0$, so $B^2$ does not get close to zero. When $b_i, b_j \geq 0$ (valid reflectance values), the gradient terms stay finite and behave as expected. You can also see a simple symmetry here: the terms involving $\alpha_{ij}$ and $b_i$ include the factor $(2\sigma_\beta b_i + \epsilon)$, while the terms involving $\beta_{ij}$ and $b_j$ include $(2\sigma_\alpha b_j + \epsilon)$.
}

{\mm
\begin{remrk}[Generalization to Signed Inputs]
\label{rmk:signed}

The formulation in Equation~\eqref{eq:nd} assumes nonnegative inputs $b_i, b_j \geq 0$, which is a reasonable fit for spectral reflectance values. But if the ND Layer is used as an intermediate layer, the inputs can come from earlier layers and may include negative activations, so we use two differentiable generalizations. Replace $b_i$ and $b_j$ in the denominator with a smooth approximation of their absolute values:

\begin{equation}
N_{ij} = \frac{\sigma_\alpha\, b_i - \sigma_\beta\, b_j}{\sigma_\alpha \sqrt{b_i^2 + \epsilon} + \sigma_\beta \sqrt{b_j^2 + \epsilon} + \epsilon}
\label{eq:nd_signed}
\end{equation}
where $\sigma_\alpha = \operatorname{softplus}(\alpha_{ij})$ and $\sigma_\beta = \operatorname{softplus}(\beta_{ij})$ as before. The function $\sqrt{x^2 + \epsilon}$ is a smooth, differentiable stand in for $|x|$. It keeps the denominator strictly positive, while the numerator still carries the sign information. The output remains bounded in $[-1, 1]$. Alternatively, one may apply a softplus activation to the inputs before the ND Layer:
\begin{equation}
N_{ij} = \frac{\sigma_\alpha\, \tilde{b}_i - \sigma_\beta\, \tilde{b}_j}{\sigma_\alpha\, \tilde{b}_i + \sigma_\beta\, \tilde{b}_j + \epsilon}, \quad \text{where } \tilde{b}_k = \operatorname{softplus}(b_k)
\label{eq:nd_softplus}
\end{equation}
This approach ensures $\tilde{b}_k > 0$ and stays fully differentiable, but it does apply a nonlinear transform to the inputs. Both options also make it straightforward to stack multiple ND Layers or plug the ND Layer into larger architectures, so the method is not limited to the first layer setting studied in this work.
\end{remrk}
}

\section{Numerical Experiments}\label{sec:experiments}
{\mm
We evaluate the proposed Normalized Difference Layer by comparing it with standard neural network architectures. The experiments measure classification accuracy, parameter efficiency, noise robustness, and how interpretable the learned coefficients are, using a real world agricultural remote sensing dataset.

\subsection{Experimental Setup}

We compiled a ground truth dataset for detecting Kochia (\textit{Bassia scoparia}), an invasive weed that is a major concern in North American agriculture. The dataset includes 2,318 labeled point samples collected across three growing seasons from 2022 to 2024 in agricultural fields near Lucklake, Saskatchewan, Canada. Each sample is labeled as either Kochia (1,071 samples, 46.2\%) or Crop (1,247 samples, 53.8\%), so it is a fairly balanced binary classification task.

For each ground truth location, we extracted spectral reflectance values from Sentinel 2 Level 2A imagery, using all 10 spectral bands available at 10m and 20m resolution: B2 (Blue), B3 (Green), B4 (Red), B5--B7 (Red Edge), B8 (NIR), B8A (Narrow NIR), B11 (SWIR-1), and B12 (SWIR-2). We picked images that lined up with the field observation dates, kept only scenes with cloud cover below 20\%, and applied the Sen2Cor atmospheric correction algorithm as preprocessing.

We compare three neural network architectures of varying depths:
\begin{itemize}[leftmargin=*]
    \item \textbf{ND Model}: The proposed architecture employing a Normalized Difference Layer as the first hidden layer. For 10 input bands, this layer computes all $\binom{10}{2} = 45$ pairwise normalized differences with learnable coefficients, followed by fully-connected layers with ReLU activations.
    
    \item \textbf{MLP Model}: A standard multilayer perceptron serving as the baseline, with hidden layer widths matched to the ND model's output dimension (45 units) to ensure a fair comparison in terms of representational capacity.
    
    \item \textbf{AttND Model}: We also test an attention gated variant where the ND layer outputs are scaled by input dependent attention weights, computed as $\mathbf{q} = \sigma(\mathbf{W}\mathbf{b} + \mathbf{c})$. Here $\mathbf{b}$ is the input band vector and $\sigma$ denotes the sigmoid function. This lets us check if letting the spectral index weights adapt to each input actually improves performance.
\end{itemize}

Network depth refers to the total number of layers including input and output. At depth 2, the ND model consists of Input$\rightarrow$ND(45)$\rightarrow$Linear(1); at depth 3, an additional hidden layer is inserted; and so forth.

All models were trained using the Adam optimizer with learning rate $\eta = 0.01$ and weight decay $\lambda = 10^{-4}$. We employed binary cross-entropy loss and batch size of 32. Training proceeded for a maximum of 150 epochs with early stopping based on validation accuracy, using a patience of 25 epochs to ensure adequate convergence time for all architectures.

Model evaluation used a rigorous 10 fold stratified cross validation protocol. In each fold, we split the data into training (70\%), validation (20\%), and test (10\%) sets, and used stratification to keep the class proportions consistent in every split. We used the validation set for early stopping and for choosing hyperparameters, and we kept the test set completely separate during training, only using it at the end for the final model comparison. All experiments were run on an NVIDIA GPU using PyTorch, with the same random seeds across runs to support reproducibility across model comparisons.

\subsection{Classification Results}

Table~\ref{tab:accuracy} presents the classification accuracy for all architectures across network depths, averaged over 10 cross-validation folds.

\begin{table}
\centering
\begin{tabular}{lccc}
\toprule
\textbf{Depth} & \textbf{ND} & \textbf{MLP} & \textbf{AttND} \\
\midrule
2 & $96.50 \pm 1.28$ & $\mathbf{97.20 \pm 0.87}$ & $96.37 \pm 1.18$ \\
3 & $\mathbf{97.15 \pm 1.20}$ & $96.94 \pm 1.22$ & $97.02 \pm 1.25$ \\
4 & $\mathbf{97.63 \pm 0.87}$ & $96.98 \pm 1.02$ & $97.28 \pm 1.15$ \\
\bottomrule
\end{tabular}
\caption{Test accuracy (\%) across 10-fold cross-validation. Bold indicates highest accuracy per depth. The ND model achieves the best performance at depths 3 and 4.}
\label{tab:accuracy}
\end{table}

The ND model gives the best accuracy at depths 3 and 4. At depth 4 it reaches 97.63\%, which is the best overall result. It also has the smallest standard deviation (0.87\%), which suggests the results stay fairly steady across the splits. At depth 2, the MLP is a bit higher (97.20\% vs.\ 96.50\%), but paired $t$ tests show the gap is not statistically significant ($p = 0.133$). At depth 4, ND is ahead of MLP by 0.65 percentage points, and the difference is close but still below the usual significance cutoff ($p = 0.081$).

\begin{figure}
    \centering
    \includegraphics[width=0.85\textwidth]{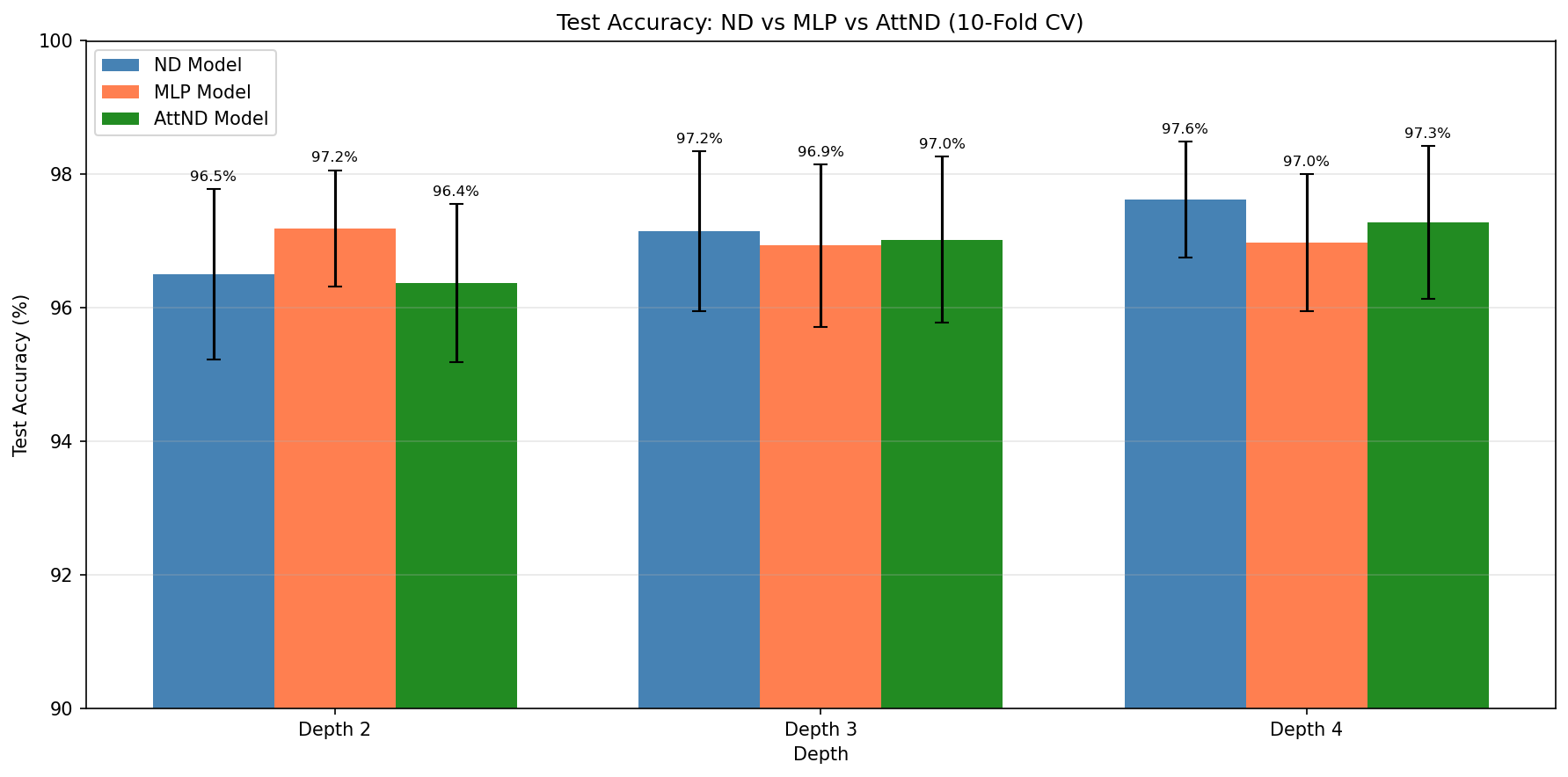}
    \caption{Test accuracy comparison across network depths. Error bars indicate standard deviation over 10 folds. All three architectures achieve comparable accuracy, with differences remaining within overlapping confidence intervals.}
    \label{fig:accuracy}
\end{figure}

The attention gated variant (AttND) did not improve on the base ND model at any depth. This suggests that, for this classification task, the fixed spectral index combinations learned during training are enough, so input dependent modulation is not really needed. This also fits the fairly homogeneous nature of the spectral discrimination problem, the cues that separate Kochia from crops stay consistent across samples. A key advantage of the proposed architecture is its parameter efficiency. Table~\ref{tab:parameters} compares the total number of trainable parameters and the resulting efficiency metric (accuracy percentage points per 100 parameters).

\begin{table}
\centering
\begin{tabular}{lcccccc}
\toprule
& \multicolumn{2}{c}{\textbf{ND}} & \multicolumn{2}{c}{\textbf{MLP}} & \multicolumn{2}{c}{\textbf{AttND}} \\
\cmidrule(lr){2-3} \cmidrule(lr){4-5} \cmidrule(lr){6-7}
\textbf{Depth} & Params & Eff. & Params & Eff. & Params & Eff. \\
\midrule
2 & 136 & \textbf{70.96} & 541 & 17.97 & 631 & 15.27 \\
3 & 2,206 & \textbf{4.40} & 2,611 & 3.71 & 2,701 & 3.59 \\
4 & 4,276 & \textbf{2.28} & 4,681 & 2.07 & 4,771 & 2.04 \\
\bottomrule
\end{tabular}
\caption{Parameter counts and efficiency (accuracy \% per 100 parameters). The ND model consistently achieves higher efficiency across all depths.}
\label{tab:parameters}
\end{table}

At depth 2, the ND model uses just 136 parameters compared to 541 for the MLP, which comes out to a 75\% reduction. Even with this much smaller footprint, the ND model reaches similar accuracy, so its parameter efficiency is nearly four times higher than the MLP baseline (70.96 vs.\ 17.97). This efficiency gain mainly comes from the structured design of the Normalized Difference Layer. Instead of learning arbitrary linear combinations, the network learns coefficients within a constrained formulation that naturally captures ratio based relationships that are known to be highly discriminative in spectral analysis.

\begin{figure}
    \centering
    \includegraphics[width=0.75\textwidth]{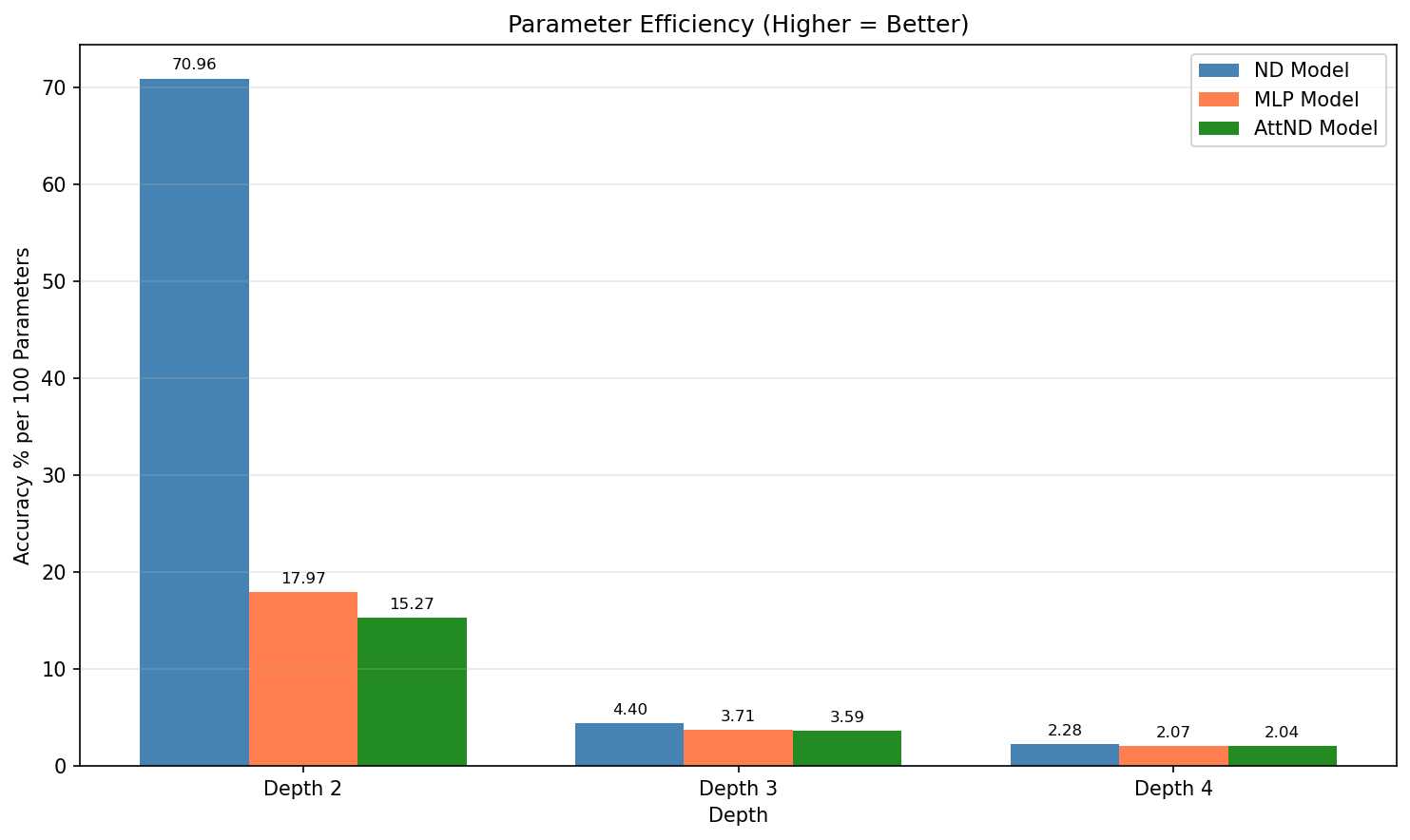}
    \caption{Parameter efficiency (accuracy per 100 parameters) across architectures. The ND model's efficiency advantage is most pronounced at depth 2, where domain-specific structure maximally constrains the hypothesis space.}
    \label{fig:efficiency}
\end{figure}

The efficiency gap gets smaller as depth increases, mainly because the extra fully connected layers account for most of the parameters. Even so, the ND model still keeps a solid advantage, which points to a real benefit from baking spectral domain knowledge into the architecture as the model becomes more complex.

\begin{figure}
    \centering
    \includegraphics[width=0.75\textwidth]{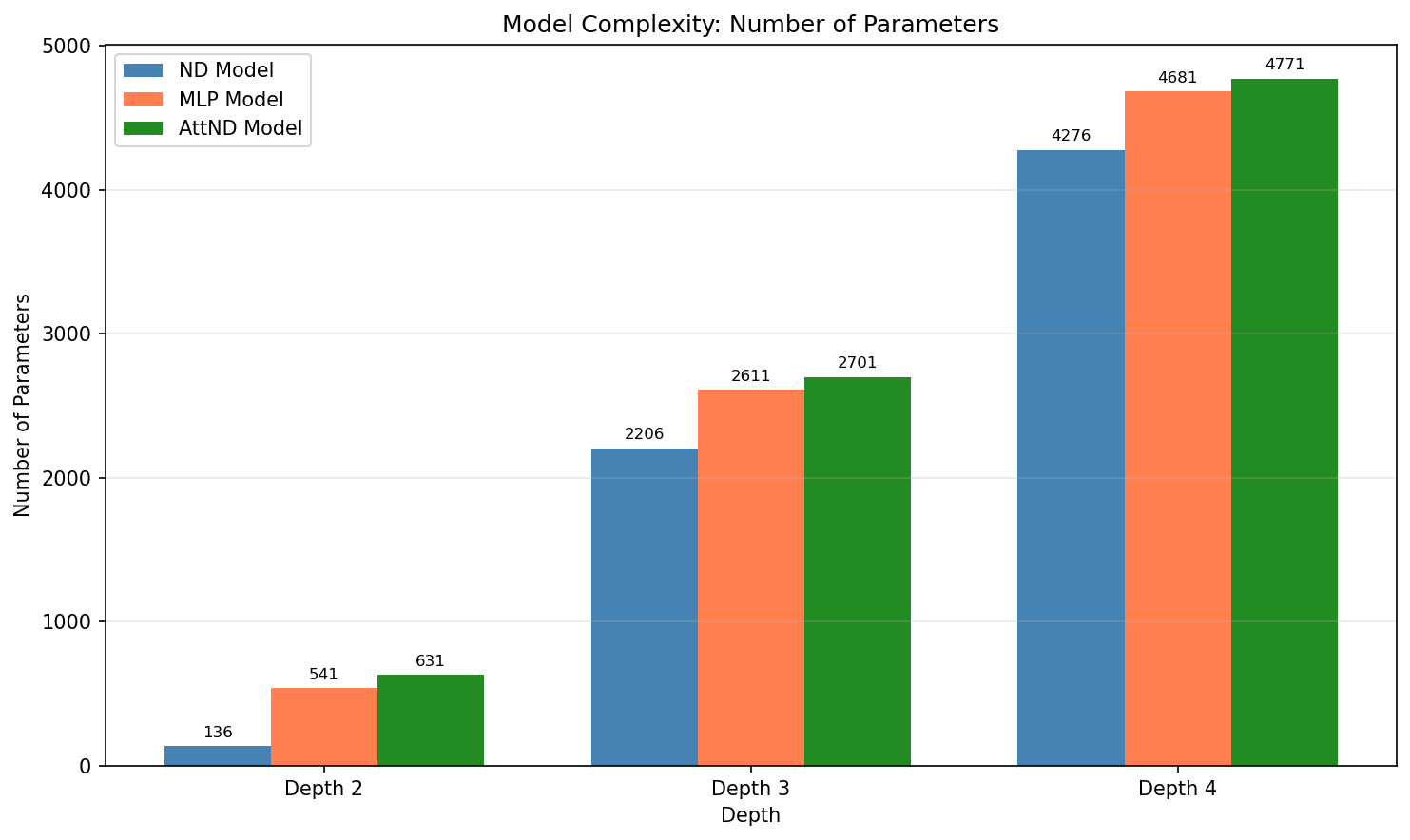}
    \caption{Total parameter count comparison. The ND architecture requires fewer parameters at all depths due to its structured first layer.}
    \label{fig:parameters}
\end{figure}

\subsection{Robustness Analysis}

Remote sensing applications have to deal with a lot of noise, like atmospheric effects, sensor calibration drift, and geometric distortions caused by changing viewing angles. We test model robustness by adding multiplicative Gaussian noise to the test inputs:

\begin{equation}
    \tilde{b}_i = b_i + \eta \cdot |b_i| \cdot z, \quad z \sim \mathcal{N}(0, 1)
\end{equation}
where $\eta$ represents the noise level as a fraction of signal magnitude. We test noise levels from 0\% to 10\%, in line with realistic Sentinel 2 uncertainty ranges that cover sensor noise (1 to 3\%), atmospheric correction residuals (3 to 5\%), and BRDF effects (5 to 10\%).

\begin{table}
\centering
\begin{tabular}{lccc}
\toprule
\textbf{Depth} & \textbf{ND Drop} & \textbf{MLP Drop} & \textbf{AttND Drop} \\
\midrule
2 & $\mathbf{0.17\%}$ & $3.03\%$ & $0.95\%$ \\
3 & $\mathbf{2.34\%}$ & $2.94\%$ & $4.85\%$ \\
4 & $5.37\%$ & $\mathbf{3.46\%}$ & $3.55\%$ \\
\bottomrule
\end{tabular}
\caption{Accuracy degradation from 0\% to 10\% noise. Lower values indicate greater robustness. The ND model exhibits remarkable stability at depth 2.}
\label{tab:noise}
\end{table}

The depth 2 ND model is very robust to noise. With 10\% noise, accuracy drops by only 0.17 percentage points, while the MLP drops by 3.03\%. This level of robustness to multiplicative noise comes straight from the normalized difference setup. If the numerator and denominator are scaled by the same factor, the ratio stays the same. Noise that affects each band separately will not cancel perfectly, but the normalized form still cuts down its impact a lot.

\begin{figure}
    \centering
    \includegraphics[width=\textwidth]{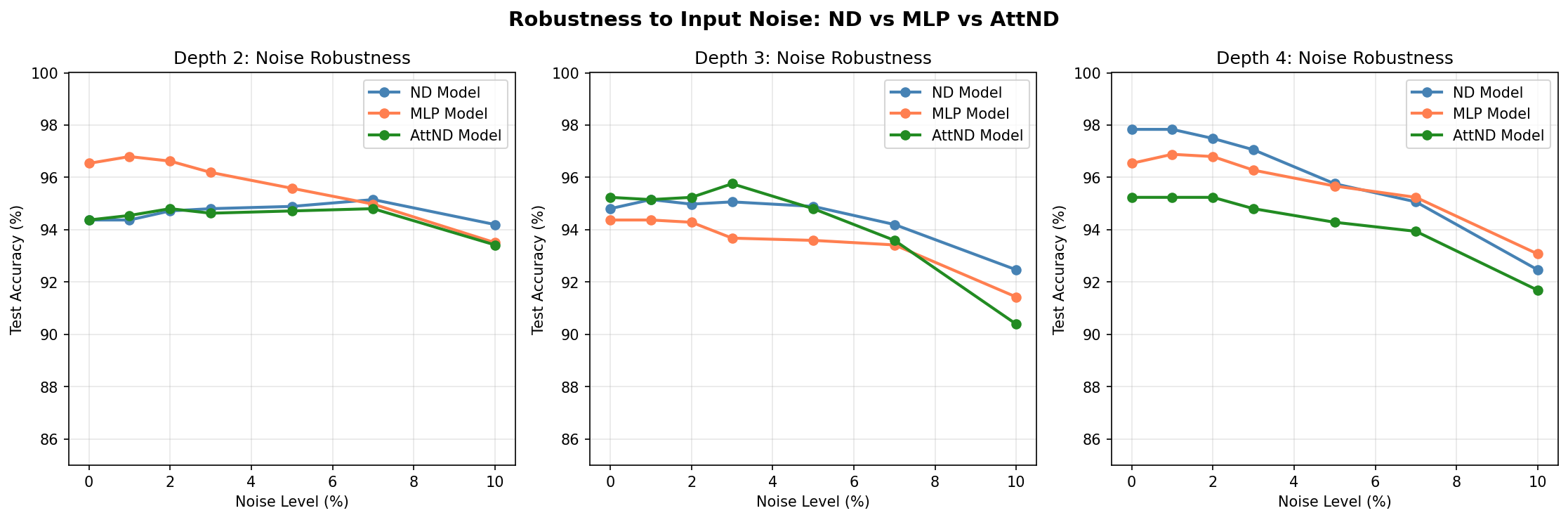}
    \caption{Accuracy under increasing input noise levels. At depth 2, the ND model maintains nearly constant performance while the MLP degrades steadily. The robustness advantage diminishes at greater depths where downstream layers introduce additional sensitivity.}
    \label{fig:noise}
\end{figure}

Figure~\ref{fig:noise} shows the ND model has the biggest robustness advantage at smaller depths. At depth 2, the ND curve is basically flat across the whole noise range, and the MLP keeps dropping as the noise goes up. That lines up with the theory. The normalized difference form gives first order protection against multiplicative noise, but once the signal passes through later unconstrained layers, that protection starts to fade.  Stability of structured computations in the presence of Gaussian noise has been studied a lot, including how perturbations change condition numbers and invertibility in high dimensional settings, in numerical analysis and random matrix theory \citep{vershynin2018high, rudelson2008invertibility, rudelson2009smallest, lotfi2022numerical, lotfi2024probabilistic}. We next examine the convergence behavior of all architectures to understand their optimization landscapes.

\begin{figure}
    \centering
    \includegraphics[width=\textwidth]{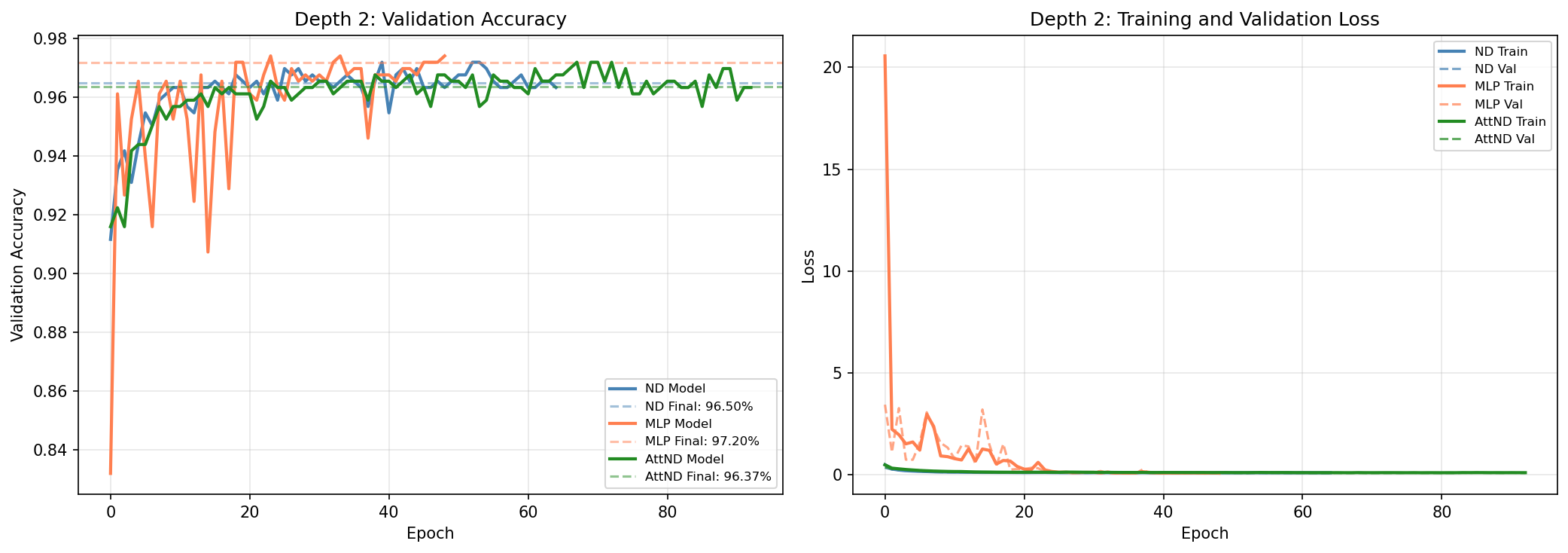}
    \caption{Training dynamics at depth 2. Left: validation accuracy over epochs. Right: training and validation loss. The ND and AttND models exhibit smooth convergence, while the MLP shows characteristic loss spikes early in training.}
    \label{fig:convergence2}
\end{figure}

\begin{figure}
    \centering
    \includegraphics[width=\textwidth]{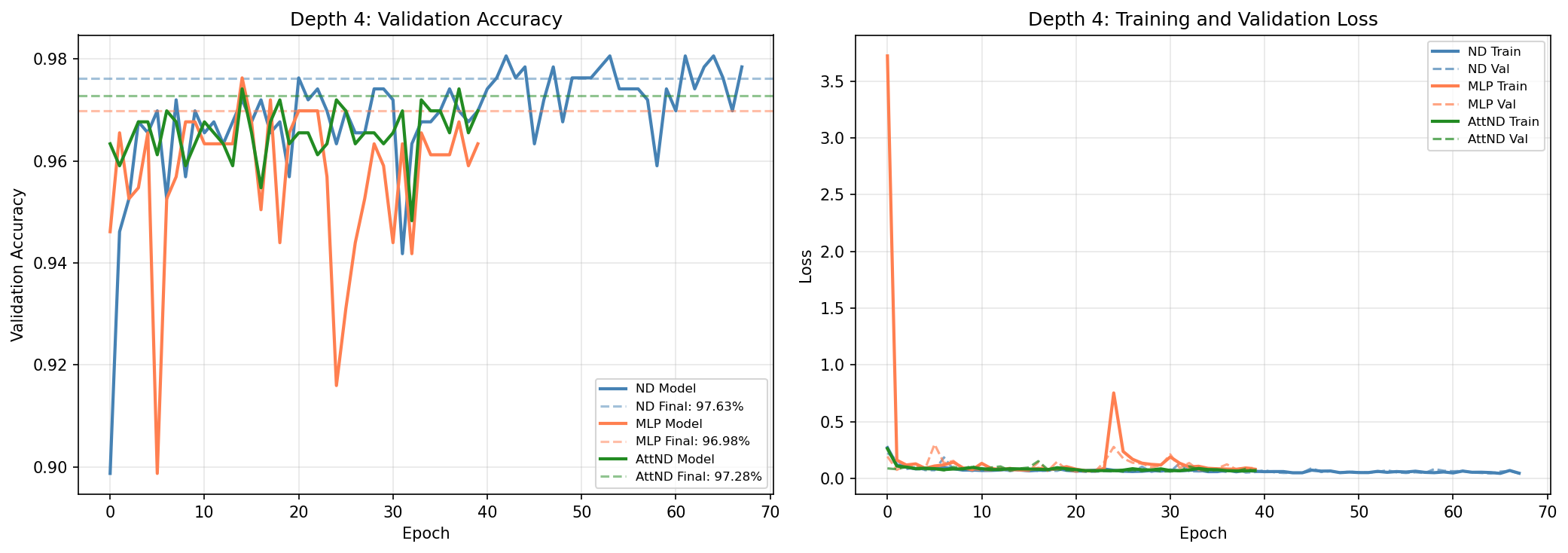}
    \caption{Training dynamics at depth 4. The ND model achieves higher peak validation accuracy and maintains stable convergence throughout training.}
    \label{fig:convergence4}
\end{figure}

The convergence plots highlight a couple of simple patterns. To start, every ND based setup tends to have a smoother loss curve than the MLP, which shows a few sharp spikes early in training. You can see those spikes in the right panels of Figures~\ref{fig:convergence2} and~\ref{fig:convergence4}, and they likely mean the MLP sometimes takes unstable gradient steps. The bounded ND layer helps keep that in check since its outputs stay in $[-1, 1]$.

Second, the ND model often reaches its best validation accuracy later than the MLP. At depth 4, the ND model kept improving through roughly epochs 40 to 77 in several folds, while the MLP usually peaked before epoch 35. That suggests the structured ND layer learns more gradually, but it often ends up in a better place overall.

\subsection{Interpretability of Learned Coefficients}

A clear advantage of the proposed architecture is that its learned parameters are easy to interpret. Each node in the ND layer learns coefficients $\alpha_{ij}$ and $\beta_{ij}$ that set how strongly bands $b_i$ and $b_j$ feed into the normalized difference computation. If we look at the ratio $\sigma_\alpha / \sigma_\beta$ (where $\sigma$ denotes the softplus function), we can see which band pairs deviate the most from the classical symmetric indices.

\begin{figure}
    \centering
    \includegraphics[width=\textwidth]{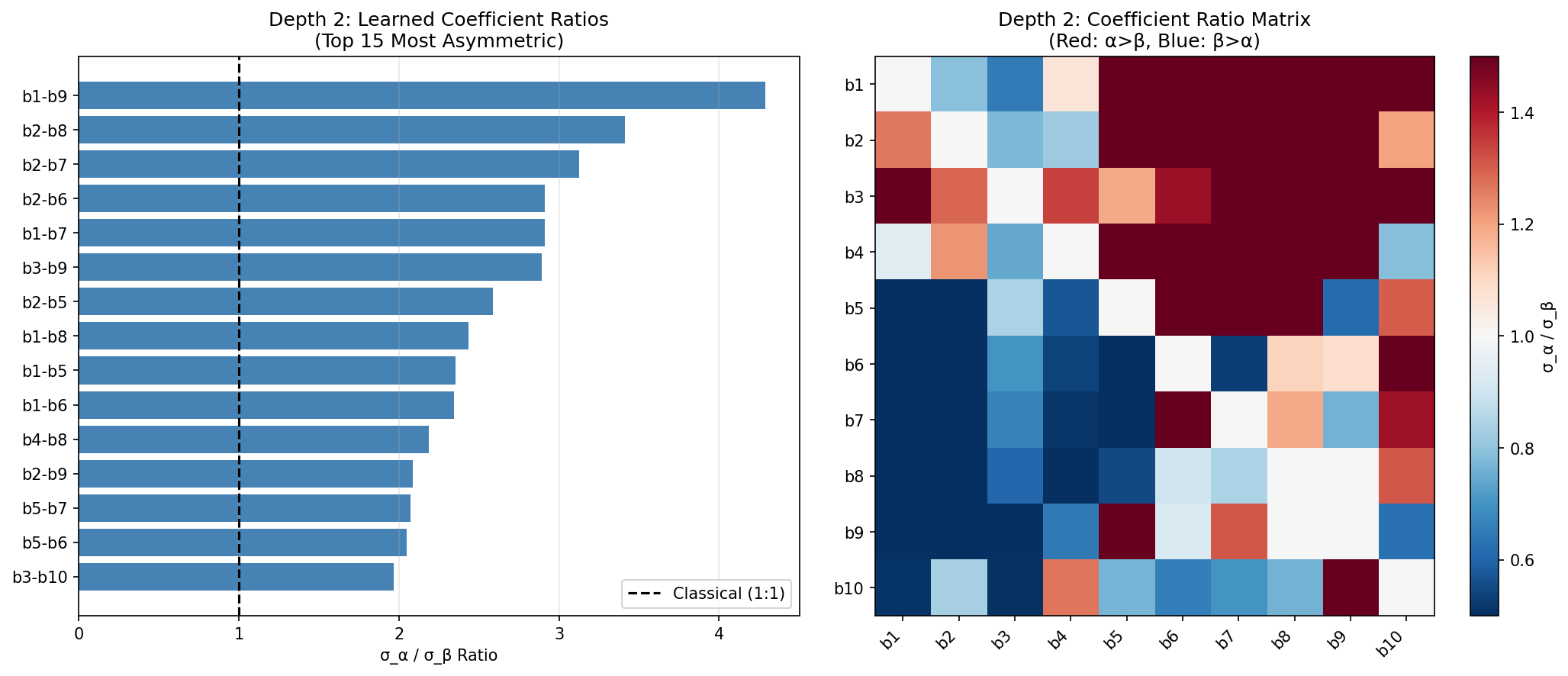}
    \caption{Learned coefficient analysis at depth 2. Left: Top 15 most asymmetric band pairs ranked by $\sigma_\alpha/\sigma_\beta$ ratio. Right: Complete coefficient ratio matrix. Red indicates $\alpha > \beta$ (first band upweighted); blue indicates $\beta > \alpha$ (second band upweighted).}
    \label{fig:weights2}
\end{figure}

Figure~\ref{fig:weights2} shows the network does not keep the usual 1 to 1 weighting. The strongest asymmetries show up when coastal aerosol (B1) or water vapor (B9) is paired with vegetation sensitive bands (B7, B8), with ratios above 4:1 in a few cases. This suggests Kochia detection works better when the normalized difference puts more emphasis on certain spectral regions than others. The coefficient matrix in the right panel shows a clear pattern. Shorter wavelength bands from B1 to B4 are often upweighted ($\alpha > \beta$, shown in red) when they are paired with longer wavelength bands from B5 to B10, and those longer wavelength bands are downweighted (blue). This imbalance seems consistent with Kochia’s spectral signature, with noticeably different reflectance in the red edge and NIR compared with typical crop canopies.

\begin{figure}
    \centering
    \includegraphics[width=\textwidth]{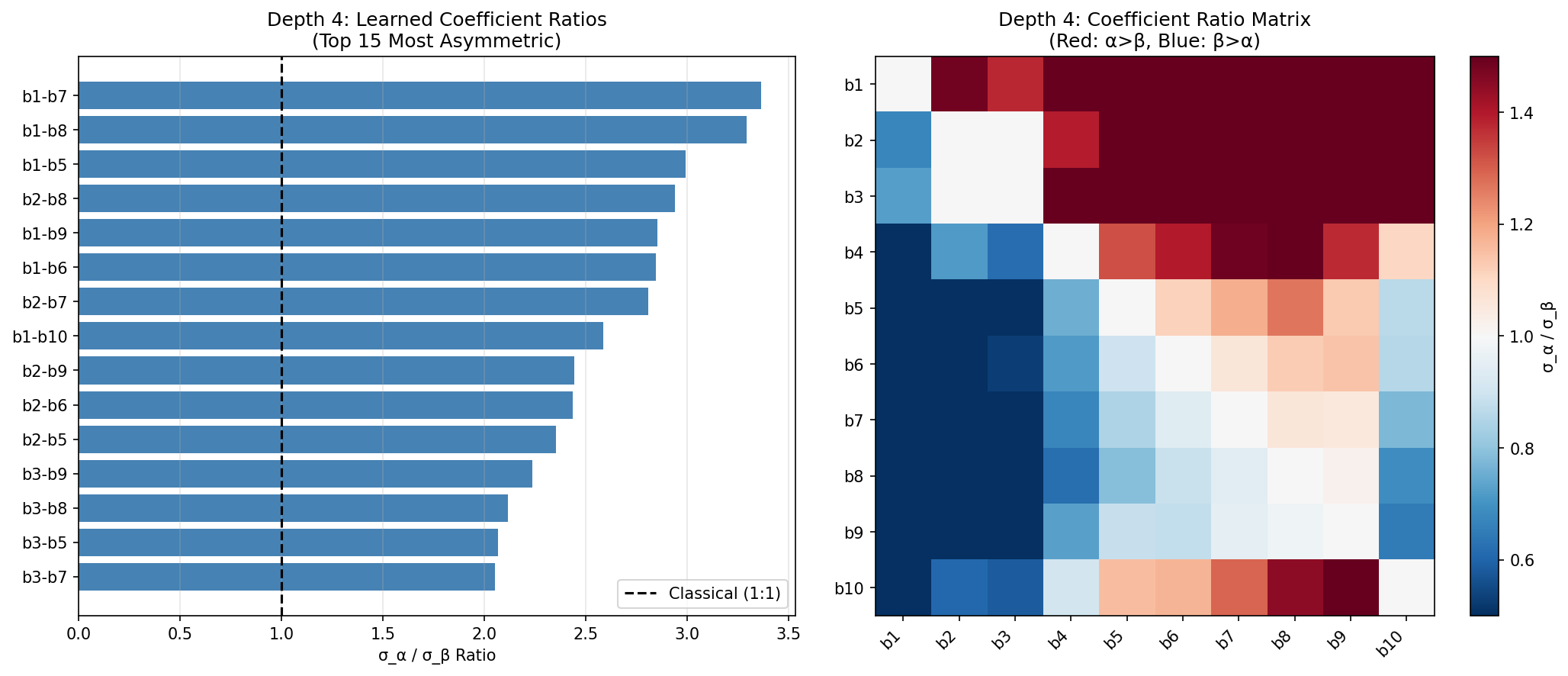}
    \caption{Learned coefficient analysis at depth 4. The asymmetric weighting patterns remain consistent across depths, indicating robust feature discovery.}
    \label{fig:weights4}
\end{figure}

The learned coefficient patterns stay consistent across network depths (see Figures~\ref{fig:weights2} and~\ref{fig:weights4}). That suggests the ND layer is capturing truly discriminative spectral relationships, not artifacts tied to one particular architectural setup. This consistency gives us more confidence that the learned weightings reflect real spectral properties of the classification problem.
}

Our experimental evaluation supports several conclusions regarding the proposed Normalized Difference Layer:

\begin{enumerate}[leftmargin=*]
    \item \textbf{Competitive accuracy}: The ND model matches or exceeds MLP performance, achieving the highest accuracy (97.63\%) at depth 4 with the lowest variance across folds.
    
    \item \textbf{Superior parameter efficiency}: At depth 2, the ND model achieves comparable accuracy with 75\% fewer parameters, resulting in nearly 4$\times$ higher parameter efficiency.
    
    \item \textbf{Exceptional noise robustness}: The shallow ND architecture exhibits near-immunity to multiplicative noise (0.17\% degradation vs.\ 3.03\% for MLP), a critical advantage for operational remote sensing applications.
    
    \item \textbf{Stable optimization}: ND-based architectures display smoother convergence with fewer loss spikes, benefiting from the bounded output range of the normalized difference formulation.
    
    \item \textbf{Interpretable representations}: The learned coefficients reveal meaningful spectral relationships that can be inspected and validated against domain knowledge, bridging the gap between black-box deep learning and interpretable spectral index design.
\end{enumerate}

As noted earlier, the AttND variant provided no improvement, suggesting simpler architectures suffice when spectral signatures are consistent within classes.

\section{Conclusion}\label{sec:conclusion}

{\mm 
In this work we introduce the Normalized Difference Layer, a neural network component built on the standard normalized difference formula, but with coefficients that the model can learn. The idea is pretty straightforward: indices like NDVI treat the two bands equally, yet there is nothing special about forcing those weights to stay at exactly 1. By letting the network learn the coefficients through gradient descent, we can find task-specific weightings while still keeping the illumination invariance and bounded outputs that make normalized differences so useful in the first place.

The mathematical setup is fairly straightforward. We derived the gradients in Proposition~\ref{prp:gradients}, which have a nice symmetric structure that makes implementation easy. Softplus keeps the learned coefficients positive, so the denominator stays away from zero and the gradients remain stable during training. When the layer is placed deeper in the network and may receive negative inputs, Remark~\ref{rmk:signed} describes two ways to extend the formulation while keeping everything differentiable. This lets the layer be stacked or embedded inside larger architectures, instead of being limited to the input.

The Kochia detection experiments give a few simple messages. First, networks that use the ND layer match or sometimes even beat standard MLPs in accuracy while using far fewer parameters. At shallow depths they need about 75\% fewer. That efficiency comes from the way the layer is set up, instead of learning any linear mix it just searches over weighted normalized differences, so the space of possibilities is much smaller. Second, the shallow ND model is quite robust to noise. With 10\% multiplicative noise, accuracy drops by only 0.17 percentage points, while the MLP loses more than 3\%. This lines up with the theory, when both bands are scaled by the same factor the ratio tends to cancel that change. Third, training is smoother and more stable, with fewer sudden spikes in the loss, because the bounded outputs keep activations from blowing up in the way they sometimes do in unconstrained networks.

What we find most interesting is that the learned coefficients remain easy to interpret. For each pair we can see how unevenly the network weights the two spectral bands, and by looking at the resulting ratios we can tell which band combinations are doing most of the work. Because these patterns stay similar across different network depths, it is a good sign that they are capturing real structure in the spectral data and not just noise from a single training run. This is useful in practice, because domain experts can check whether the learned weightings line up with what they already know about the underlying spectroscopy.

The broader point here is how we bring domain knowledge into neural network design. The broader point here is how we bring domain knowledge into neural network design. In many practical pipelines we just compute a spectral index up front, for example NDVI or another fixed formula, and then pass that single number into a generic classifier. This can work well, but it freezes the coefficient choices and gives the network no chance to refine them. In contrast, our approach keeps the index formula inside the model as a differentiable layer whose parameters are learned from data. That way we keep the structure that makes the formula useful while still allowing it to adapt to the task at hand. The same idea could carry over to other domains where well established functional forms encode prior knowledge, whether in physics, chemistry, ecology, or elsewhere.
From a theoretical point of view, this approach shows how building domain structure into neural architectures can act as a kind of implicit regularization. Recent work makes this precise, showing that continuity together with coercive regularization ensures the existence of minimizers and gives uniform control of excess risk by keeping optimization inside compact sets in parameter or function space \citep{meysami2025deep}. The normalized difference layer fills a comparable role, mainly because its outputs stay bounded and its parameters are arranged in a simple, structured way. This suggests that the classical tools of remote sensing, such as ratio based indices with illumination invariance, give more than empirical convenience. They also impose geometric constraints that support stable training and more generalizable learning.

There are several directions to explore. Another direction is to test how well the learned coefficients carry over to new sensors or different regions, which would show whether the weightings are capturing something fundamental in the spectral signatures or are mostly tuned to this specific dataset. It would also be interesting to stack multiple ND layers or combine them with spatial processing for applications where local context matters. The code for the layer is simple enough that others can plug it into their own models and try these ideas on their data without a lot of extra work.
}

\section*{Acknowledgements}

This work was supported by the Saskatchewan Ministry of Agriculture through the Agriculture Development Fund [project number 20230164].


\bibliographystyle{apalike}
\bibliography{bibfile}
\end{document}